	\documentclass{article}
	\usepackage{spconf,amsmath,graphicx}
	
	\usepackage{times}
	\usepackage{epsfig}
	\usepackage{graphicx}
	\usepackage{amsmath}
	\usepackage{amssymb}
	\usepackage{booktabs}
	\usepackage{hyperref}
	\usepackage{amsthm}

	\newcommand{\R}{\mathbb{R}}
	
	\newcommand{\G}{\mathcal{G}}

	\usepackage{multirow}
	\newcommand{\Z}{\mathbb{Z}}
	\newcommand{\yaps}{y_{\text{APS}}}

	\DeclareMathOperator*{\argmax}{argmax}

	\newcommand{\aps}{\text{APS}}

	\newcommand{\ti}{\widetilde{i}}
	
	\newcommand{\tz}{\tilde{z}}
	
	\newcommand{\yk}{y^{(k)}}
	\newcommand{\Yk}{Y^{(k)}}

	\newtheorem{proposition}{Proposition}

	\newcommand{\w}{\omega}

	\usepackage{blindtext}
	
	\newcommand{\DAt}{D^A_2}
	\newcommand{\UAt}{U^A_2}

	\newcommand{\tx}{\tilde{x}}
	\newcommand{\ty}{\tilde{y}}
	
	\newcommand{\xl}{x^{(l)} }

	\newcommand{\Fl}{F^{(l)} }
	\newcommand{\Sl}{s^{(l)} }

	\newcommand{\sz}{s^{(0)}}

	\newcommand{\tsz}{\tilde{s}^{(0)}}
	
	\newcommand{\xo}{{x}^{(1)}}
	\newcommand{\txo}{\tilde{x}^{(1)}}
	
	\newcommand\blfootnote[1]{%
		\begingroup
		\renewcommand\thefootnote{}\footnote{#1}%
		\addtocounter{footnote}{-1}%
		\endgroup
	}
	
	\title{Truly shift-equivariant convolutional neural networks with adaptive polyphase upsampling}
	\twoauthors
	{Anadi Chaman}
	{University of Illinois at Urbana-Chamaign\\achaman2@illinois.edu}
	{Ivan Dokmanić}
	{University of Basel\\ivan.dokmanic@unibas.ch}
	
\begin{document}
		%\ninept
		%
		\maketitle
		\thispagestyle{plain}
		\pagestyle{plain}

		\begin{abstract}
			Convolutional neural networks lack shift equivariance due to the presence of downsampling layers. In image classification, adaptive polyphase downsampling (APS-D) was recently proposed to make CNNs perfectly shift invariant. However, in networks used for image reconstruction tasks, it can not by itself restore shift equivariance. We address this problem by proposing adaptive polyphase upsampling (APS-U), a non-linear extension of conventional upsampling, which allows CNNs with symmetric encoder-decoder architecture (for example U-Net) to exhibit perfect shift equivariance. With MRI and CT reconstruction experiments, we show that networks containing APS-D/U layers exhibit state of the art equivariance performance without sacrificing on image reconstruction quality. In addition, unlike prior methods like data augmentation and anti-aliasing, the gains in equivariance obtained from APS-D/U also extend to images outside the training distribution.\blfootnote{Code available at \url{https://github.com/achaman2/truly_shift_invariant_cnns}.}
		\end{abstract}
		\section{Introduction}

		In image-to-image regression problems like MRI and CT reconstruction, shifts in input to a convolutional neural network (CNN) should result in similar shifts in the network's output. This property, called shift equivariance, is a highly desirable inductive bias that allows networks to reconstruct objects accurately irrespective of their position in the frame. However, recent works have shown that despite the presence of convolutions which are shift equivariant, the output of a CNN can be very unstable to shifts in its input \cite{Azulay_Weiss,engstrom2019a}. This is because of the use of downsampling layers that lack translation equivariance. For example, Fig. \ref{fig:intro_downsampling}(a) shows that shifting a signal can significantly change its downsampled output.
		
		\begin{figure}[t]
			\includegraphics[width=\linewidth]{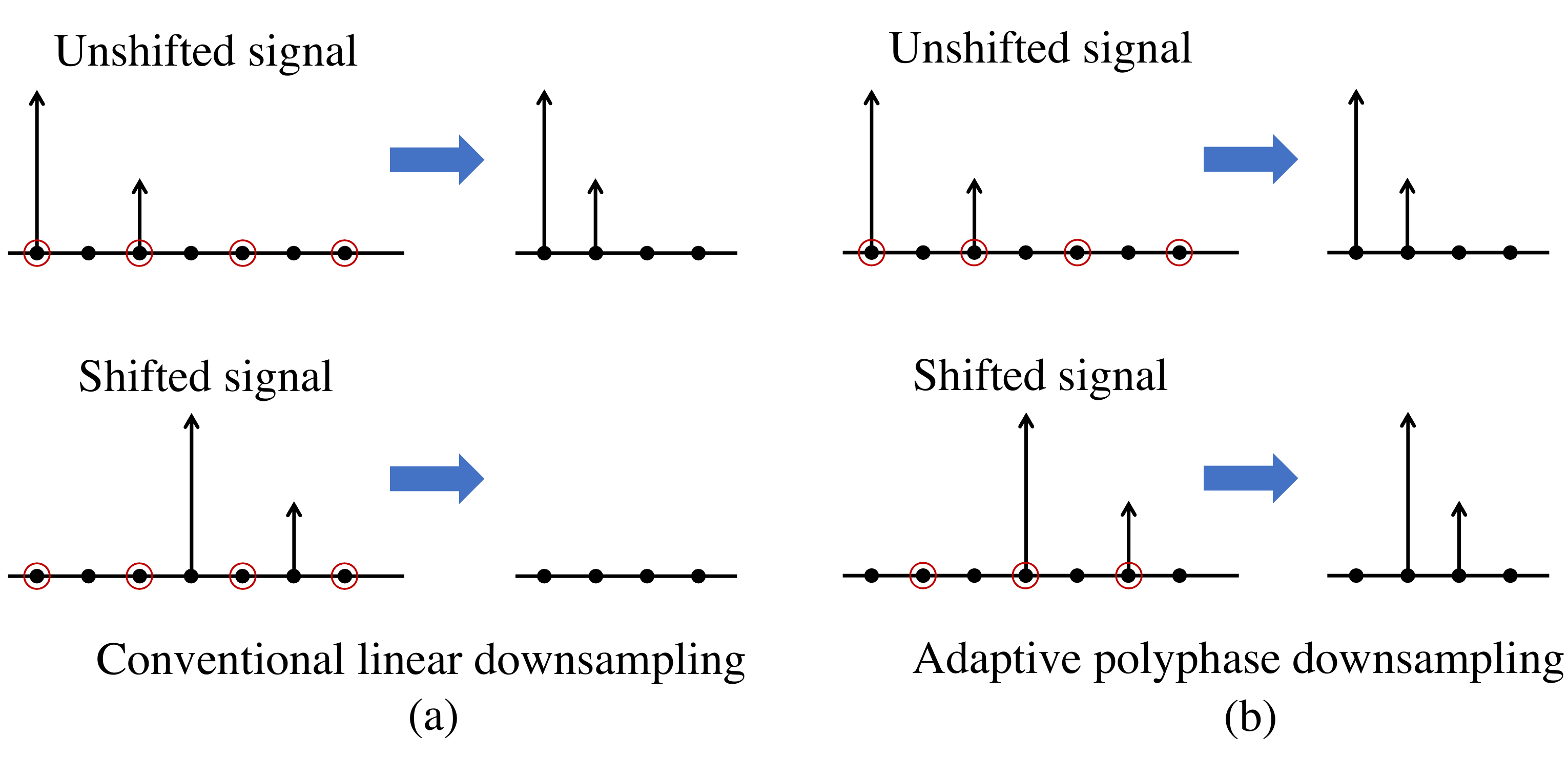}
			\centering
			\caption{(a) Conventional downsampling selects pixels along fixed locations on the sampling grid (shown with red circles). A shift in the signal changes pixels intensities on this fixed grid resulting in a very different output. (b) Adaptive polyphase downsampling selects the sampling grid adaptively such that a shift in input translates the downsampled output.}
			\label{fig:intro_downsampling}
		\end{figure}

		This problem has recently received attention in the context of shift invariance in image classification. Methods like data augmentation \cite{Simonyan15} and anti-aliasing \cite{Zhang19,DDAC,Azulay_Weiss} have been used to improve robustness to shifts. However, these methods have limitations and do not enable perfect invariance \cite{Azulay_Weiss,chaman2020truly}. In our recent work \cite{chaman2020truly}, we proposed adaptive polyphase downsampling (APS-D), a non-linear sampling scheme that allows CNN classifiers to exhibit perfect shift invariance. As shown in Fig. \ref{fig:intro_downsampling}(b), by selecting its sampling grid in a signal dependent manner, APS-D consistently downsamples the same signal structures irrespective of any shift in input. 
		
		 While APS-D enables shift invariance in classification, it can not by itself restore shift equivariance in CNNs used for image reconstruction tasks. This is because unlike shift invariance, where the goal is to obtain identical outputs for any shift in input, shift equivariance requires to propagate spatial location information from the network's input and output. Since APS-D reduces signal resolution on downsampling, shifts in the input which are not integer multiples of its stride can be mapped to the same output. Consequently, the shifts in input are lost when these low resolution feature maps pass through upsampling layers to produce the final output. To address this challenge, we propose adaptive polyphase upsampling  (APS-U), a non-linear extension of classical upsampling, that allows CNNs with symmetric encoder-decoder structures (eg. the popular U-Net architecture \cite{unet}) to exhibit perfect shift equivariance. APS-U achieves this by upsampling feature maps onto the same grid which was selected by APS-D for sampling in the encoder. 
		 
		 Using MRI and CT reconstruction experiments with U-Net, we show that our approach provides state-of-the art results in shift equivariance without sacrificing image reconstruction performance.  While methods like anti-aliasing and data augmentation do improve robustness to shifts on average, there still exist shifts that can adversely impact reconstruction performance. This is a serious problem for applications like medical imaging which require strong robustness guarantees for the reconstructed images. These methods have also been shown to get adversely impacted by non-linear activations like ReLU \cite{chaman2020truly}. On the other hand, our approach is highly robust to shifts and is not impacted by any point-wise non-linearity. In addition, our experiments reveal that networks with APS-D/U continue to remain shift equivariant on images outside the training distribution, which is not the case for anti-aliasing and data augmentation. Similar to APS-D, adaptive polyphase upsampling also does not require additional learnable parameters and can be easily integrated into existing architectures. \\

		\noindent
		\textbf{Related work.} The success of convolutional neural networks inspired research on embedding equivariances to more complex transformations like rotations, scale, reflections and the action of arbitrary groups \cite{pmlr-v48-cohenc16,pmlr-v97-cohen19d,s.2018spherical,Weiler_2018_CVPR,Sosnovik2020Scale-Equivariant,pmlr-v70-ravanbakhsh17a}. Similarly, the related problem of invariance to deformations has been theoretically studied with kernel methods \cite{10.5555/3322706.3322731,10.5555/3295222.3295369} and wavelet filter banks \cite{group_invariant_scattering,6522407}. However, the impact of downsampling on the stability of CNNs to shifts has only recently been analyzed \cite{Azulay_Weiss,landscape_of_spatial_robustness}. Manfredi and Wang \cite{manfredi2020shift} assessed the lack of shift equivariance in CNN architectures used for object detection. Anti-aliasing and data augmentation were shown to improve shift invariance in classification \cite{Azulay_Weiss, Zhang19, DDAC}. However, the improved robustness to shifts from these methods does not extend to image patterns not seen during training \cite{Azulay_Weiss}. The gains in shift invariance obtained from anti-aliasing were also shown to be limited by the action of non-linear activations in the network \cite{Azulay_Weiss, chaman2020truly}. These challenges were addressed in our earlier work \cite{chaman2020truly} where we proposed a new downsampling method called adaptive polyphase sampling that enabled perfect shift invariance in CNN classifiers.

		\section{Background}
		\label{sec:background}
		%\subsection{Equivariance to transformations.}
		\textbf{Shift equivariance.} Let $x\in \R^d$ be an input to an operator $\G$, and $T_k$ represent translation by $k \in \R^d$. Then $\G$ is said to be equivariant to shifts if $\G(T_k(x)) = T_k(\G(x))$.\\

		\noindent
		\textbf{Sampling and shift equivariance.} Let $D_2$ and $U_2$ denote downsampling and upsampling\footnote{Stride-2 max-pooling layers used in modern CNNs can be decomposed into a dense (stride 1) max-pool followed by $D_2$. Similarly, transposed convolutions and nearest-neighbour upsampling layers can be characterized in terms of $U_2$.} operations with stride $2$. For a 1-D signal $x(n)$, $y = D_2(x)$ is given by $y(n) = x(2n)$. $U_2$ upsamples signal $y$ as 
		\begin{equation}
		U_2(y) = z(n) = \begin{cases}
		y(n/2), \text{ when }n\text{ is even},\\
		0, \text{ otherwise.}
		\end{cases}
		\end{equation}
		Let $T_k = x(n-k)$ represent a $k$-pixel shift in $x$. For an odd shift $k = (2m+1)$ with $m\in \Z$, $D_2$ satisfies $D_2(T_{2m+1}(x)) = T_mD_2(T_1(x))$. It is therefore, unsurprisingly, not shift equivariant. Similarly, $\forall k \in \Z$, the action of $U_2$ on $T_k(x)$ is given by
		\begin{equation}
		\label{eq:u2_shift}
		U_2(T_k(x)) = T_{2k}(U_2(x)).
		\end{equation}
		The lack of equivariance in $D_2$ can not be corrected by anti-aliasing. Even if $x$ and $T_k(x)$ were ideal low pass filtered before downsampling, the subsequent sampled outputs $y_a$ and $\yk_a$ would have DTFTs 
		\begin{equation}
		Y_a(\w)= \frac{1}{2}X\Big(\frac{\w}{2}\Big) \text{ and } \Yk_a(\w) = \frac{1}{2}X\Big(\frac{\w}{2}\Big)e^{-\frac{jk \w}{2}}.
		\end{equation}
		
		\noindent
		One can observe that for any odd shift $k$, there does not exist an $n_0\in \Z$ such that $\Yk_a(\w)  = Y_a(\w)e^{-jn_0\w}$. This indicates that $\yk_a$ can not be obtained by translating $y_a$ with any integer shift. \\

		\noindent
		\textbf{Adaptive polyphase downsampling}. Originally proposed in \cite{chaman2020truly} to make CNN classifiers shift invariant, adaptive polyphase downsampling (APS-D) samples a 1-D signal $x$ by considering 2 possible sampling grids (illustrated in Fig. \ref{fig:intro_downsampling}(b)). The signals supported on these grids are called polyphase components and are given by $y_0(n) = x(2n)$ and $y_1(n) = x(2n+1)$. APS-D chooses the polyphase component of $x$ with the highest $l_p$ norm as its downsampled output $\DAt(x)$,  i.e. $\yaps =\DAt(x) = y_i$ where $i = \argmax_j \{\lVert y_j \rVert_p \}_{j=0}^1$. We show in \cite{chaman2020truly} that shifting the input of APS-D always results in a shift in its output. More formally, if $\yaps = \DAt(x)$ and $\yk_\aps = \DAt(T_k(x))$, then
		\begin{equation}
		\yk_\aps = \begin{cases}
		T_{\frac{k}{2}}(\yaps), \text{ when $k$ is even},\\
		T_{\frac{k+2i-1}{2}}(\yaps), \text{ when $k$ is odd},
		\end{cases}
		\label{eq:yaps_def}
		\end{equation}
		where $i\in \{0,1\}$ represents the polyphase component of $x$ with the highest norm. From \eqref{eq:yaps_def}, since $\DAt(T_k(x)) \neq T_k(\DAt(x))$ and the shift between the two signals is dependent on $i$, APS-D is not shift equivariant in the usual sense. Yet it is superior to classical downsampling which can not guarantee a shift in its output when its input is translated. We call this weaker condition as $\sigma$-equivariance.

		\section{Our approach}
		\label{sec:our_approach}
		We will focus our discussion on shift equivariance in U-Net, a highly popular architecture used for image reconstruction tasks \cite{unet_inverse_problems,Barbastathis:19,zbontar2018fastMRI}. The general conclusions can be extended to similar architectures containing symmetric encoder-decoder structure. The U-Net's encoder takes a signal $x$ as input and generates a multi-scale decomposition with $L$ scales. Feature map $\xl_e$ at scale $l$ is used by a convolutional shift equivariant block $\Fl_e$ to generate $\Sl_e = \Fl_e(\xl_e)$ which is then downsampled to produce $x^{(l+1)}_e = D_2(\Sl_e)$ for the next scale. Similarly, at scale $l$ in the decoder, a convolutional $\Fl_d$ generates $\xl_d$ from upsampled feature map $\Sl_d = U_2(x_d^{(l+1)})$ and skip connection $\Sl_e$.
		
		By replacing the downsampling layers of U-Net with APS-D, we can make its encoder $\sigma$-equivariant. Then, a shift $k$ in input  shifts feature maps $\{x_e^{(l)}\}_{l=1}^L$ by $\{k_l\}_{l=1}^L$, where the relation between $k_l$ and $k_{l-1}$ can be obtained via \eqref{eq:yaps_def}.

		\subsection{A case for adaptive upsampling}
		\label{sec:need_for_adaptive_upsampling}
		
	%	While APS-D yields $\sigma$-equivariance in the encoder, it can not by itself make the U-Net architecture perfectly shift equivariant. This is because of ....
	
	%Replacing the downsampling layers of a U-Net with APS-D can not by itself restore shift equivariance. 
	
	While APS-D can make a U-Net's encoder $\sigma$-equivariant, it can not restore perfect shift equivariance in the overall U-Net architecture. To understand this better, consider a toy example with a signal $x$ and its $1$-pixel shift $\tx = T_1(x)$ as shown in Fig. \ref{fig:aps_up_case}(a). Downsampling them with APS-D produces $\yaps = \DAt(x)$ and $\ty_\aps = \DAt(\tx)$ which correspond to polyphase components with indices $i=0$ and $1$ respectively. Notice from Fig. \ref{fig:aps_up_case}(b) that despite being sampled from different spatially located grids, $\yaps$ and $\ty_\aps$ are identical and the 1-pixel shift between the inputs is lost. Therefore, upsampling them directly results in outputs $z$ and $\tilde{z}$ which are not 1-pixel shifted versions of each other (Fig. \ref{fig:aps_up_case}(c)).
		
	One way to counter this is by upsampling the signals back to the grid from which they were originally sampled. For example, in Fig. \ref{fig:aps_up_case}(c), $\tz$ obtained after upsampling $\ty_\aps$, can be translated to the grid with index $i=1$. Similarly, $z$ can be shifted to $i=0$. One then obtains outputs $z_\aps$ and $\tilde{z}_\aps$ which from Fig. \ref{fig:aps_up_case}(d) have the same shift between them as the original inputs $x$ and $\tilde{x}$. 
		
	The above example shows that by using the index of polyphase component used by APS-D for downsampling, one can upsample signals to the `correct' sampling grids, and preserve spatial location information. We formalize this notion in the next section.

			\begin{figure}[t]
			\includegraphics[width=\linewidth]{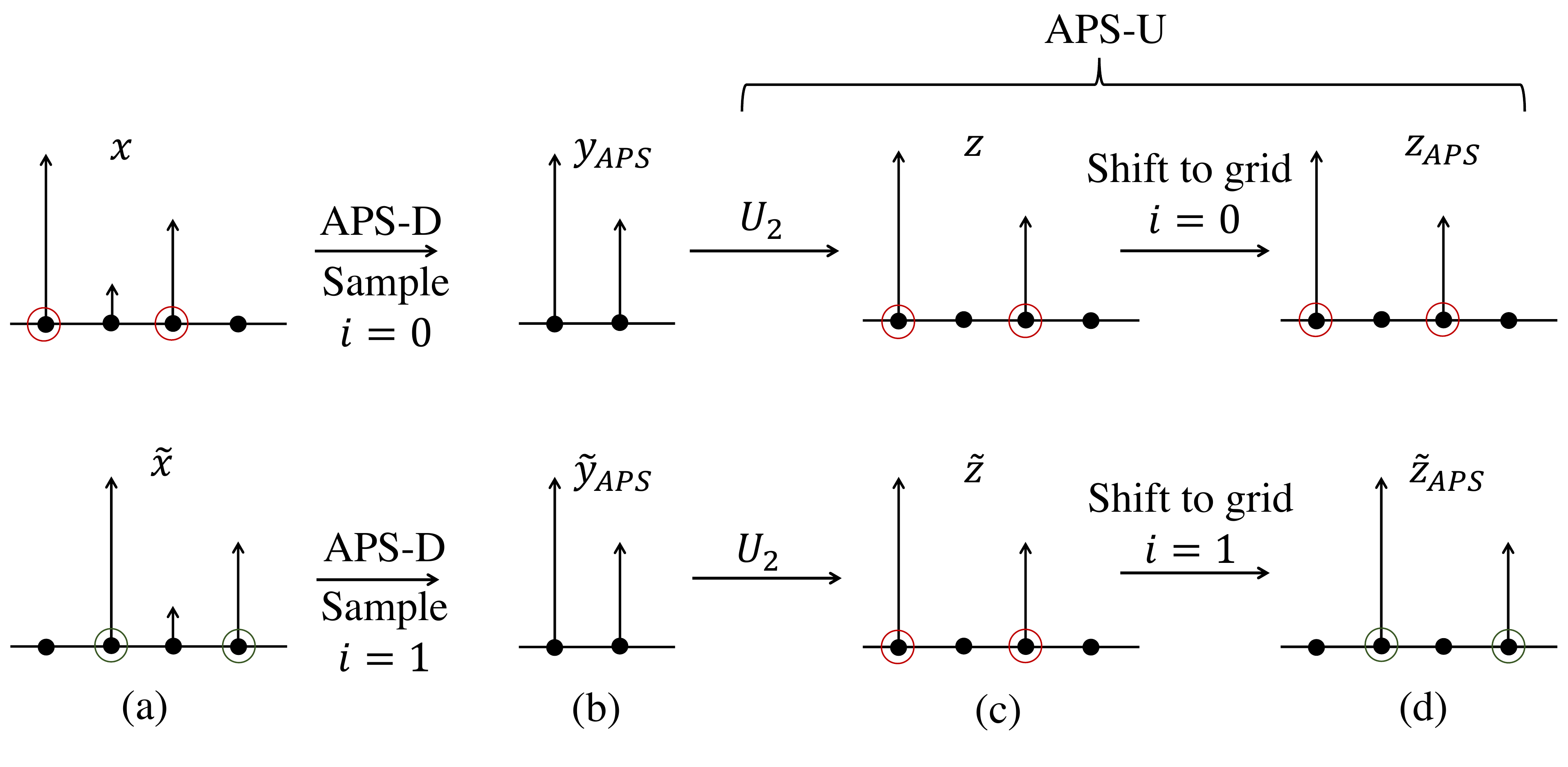}
			\centering
			\caption{(a)-(b) The relative shift between $x$ and its 1-pixel shift $\tx$ is lost after they are downsampled. (c) Outputs of conventional upsampling, therefore, are not 1-pixel shifted versions of each other. (d) By shifting the upsampled signal to the grid originally chosen by APS-D, shift equivariance is achieved. }
			\label{fig:aps_up_case}
		\end{figure}

		\vspace{-1mm}
		\subsection{Adaptive polyphase upsampling}
		\label{sec:aps_up}
		
		For simplicity, we will focus our discussion to sampling of 1-D signals with stride 2. The conclusions can be easily generalized to higher dimensions and stride. Let $\{y_i\}_{i=0}^1$ be the polyphase components of a 1-D signal $x$, given by $y_i(n) = D_2(T_{-i}(x)) = x(2n+i)$, where $i\in \{0, 1\}$. We denote $i_x$ to be the index of polyphase component with the highest norm. Then the downsampled output obtained from APS-D is $\yaps = \DAt(x) = D_2(T_{-i_x}(x))$. Let $\UAt$ denote adaptive polyphase upsampling  (APS-U) operator with stride 2. Upsampling $\yaps$ with APS-U yields $\UAt(\yaps, i_x)$ which is defined as
		\begin{equation}
		\label{eq:y_apsup_def}
		\UAt(\yaps, i_x) = T_{i_x}\Big(U_2(\yaps)\Big).
		\end{equation}
		
		The example in Fig. \ref{fig:aps_up_case} illustrates that $\UAt (\DAt(x))$ is equivalent to zeroing out all pixels of $x$ except the ones located on the grid $i_x$. Since APS-D selects sampling grid in a manner consistent to translations, one can expect a shift in $x$ to result in the same shift in $\UAt (\DAt(x))$. Indeed, we show in Proposition \ref{prop:shift_eq_UD} that $\UAt \circ \DAt$ is a shift equivariant operator.
		
		\begin{proposition}
			\label{prop:shift_eq_UD}
			Let $\DAt$ and $\UAt$ denote APS-D and APS-U operators with stride 2. Then $\UAt \circ \DAt$ is shift equivariant, i.e.
			\begin{equation}
			\UAt \circ \DAt(T_k(x)) = T_k(\UAt \circ \DAt(x)), \text{ }\forall k \in \Z.
			\end{equation}
		\end{proposition}
		
		\begin{proof}
			Let $x$ and its $k$-pixel shift $\tx = T_k(x)$ be inputs to $\UAt \circ \DAt$. Without loss of generality, we assume $k$ to be an odd integer. Let $i_x$ and $\ti_x$ be the polyphase component indices of $x$ and $\tx$ respectively with the highest norm. Then 
			\begin{equation}
			\DAt(x) = D_2(T_{-i_x}(x))\text{ and }\DAt(\tx) = D_2(T_{-\ti_x}(\tx)).
			\end{equation}
			From \cite{chaman2020truly}, shifting a signal with an odd shift $k$, permutes the order of its polyphase components. Therefore, $\ti_x = 1-i_x$. We can now write 
			\begin{equation}
			\label{eq:ua_da_x}
			\UAt \circ \DAt(x) = T_{i_x}U_2D_2(T_{-i_x}(x)).
			\end{equation}
			Similarly, 
			\begin{align}
			\UAt \circ \DAt(\tx) &=  T_{\ti_x}U_2D_2(T_{-\ti_x}(\tx))\\
			&=T_{\ti_x}U_2D_2(T_{k-\ti_x}(x))\\
			\label{eq:k_pop_out}
			&=T_{\ti_x} T_{k-1}U_2D_2(T_{1-\ti_x}(x))\\
%			\label{eq:ix_tix_relation}
%			&=T_{k-i_x}U_2D_2(T_{i_x}(x))\\
			%										&=T_{k-i_x} T_{2i_x}U_2 D_2 (T_{-i_x})\\
			\label{eq:shift_pop_out}
			&= T_k(T_{i_x}U_2 D_2 (T_{-i_x})),
			\end{align}
			where \eqref{eq:k_pop_out} and \eqref{eq:shift_pop_out} follow from identities \begin{align*}
			D_2(T_{2m+1}(x)) &= T_m(D_2(T_1(x))),\\ 
			U_2(T_m(x)) &= T_{2m}(U_2(x))),
			\end{align*} and $\ti_x = 1-i_x$ (for odd $k$). From \eqref{eq:ua_da_x} and \eqref{eq:shift_pop_out}, $\UAt \circ \DAt(T_k(x)) = T_k(\UAt \circ \DAt(x))$. The result can similarly be shown for even $k$ in which case $\ti_x = i_x$.  
		\end{proof}

		\subsection{Restoring shift equivariance with APS-U}
		\label{eq:restoring_equivariance}
		
		We saw in Section \ref{sec:aps_up} that the composition of APS-U and APS-D given by $\UAt \circ \DAt$ is shift equivariant. We will now show that this property allows a U-Net containing APS-D/U layers to be perfectly shift equivariant. 
		
		\begin{proposition}
			\label{prop:unet_prop}
			A U-Net architecture with downsampling and upsampling layers replaced by APS-D and APS-U layers is shift equivariant.
		\end{proposition}

		\begin{proof}
			We first prove the claim for a U-Net containing a single down and upsampling layer, i.e., with $L=1$ as shown in Fig. \ref{fig:single_layer_unet_aps_d_u}. Equivariance for higher $L$ follows by induction.
			
			Let input $x$ to the network produce output $y$ and internal feature maps $\sz_e$, $\xo_e$, $\xo_d$ and $\sz_d$. Similarly, the output and feature maps for $T_k(x)$ are denoted by $\ty$ and $\tsz_e$, $\txo_e$, $\txo_d$ and $\tsz_d$. Since $F_e^{(0)}$ is convolutional, $\tsz_e = T_k(\sz_e)$. Let $i_x$ and $\ti_x$ be the indices of polyphase components of $\sz_e$ and $\tsz_e$ with the highest norm. Then, 
			\begin{align}
			\xo_e = \DAt(\sz_e) = D_2(T_{-i_x}(\sz_e))\\
			\txo_e = \DAt(\tsz_e) = D_2(T_{-\ti_x}(\tsz_e))
			\end{align}
			From the proof of Proposition \ref{prop:shift_eq_UD}, $\ti_x = 1-i_x$ when $k$ is odd and $\ti_x = i_x$ for even $k$. For odd $k$, we have 
			\begin{align}
			\sz_d = \UAt(F_e^{(1)}\xo_e )&= T_{i_x}U_2(F_e^{(1)}\xo_e)\\
			\label{eq:sz_formula}
			&=T_{i_x}U_2(F_e^{(1)}D_2(T_{-i_x}(\sz_e)))
			\end{align}
			Similarly,
			\begin{align}
			\tsz_d = \UAt(F_e^{(1)}\txo_e )= T_{\ti_x}U_2(F_e^{(1)}D_2(T_{k-\ti_x}(\sz_e))).
			\end{align}
			Using the shift equivariance of $F_e^{(1)}$ and arguments similar to the proof of Proposition \ref{prop:shift_eq_UD}, one can show that 
			\begin{align}
			\tsz_d =T_{k + i_x} U_2(F_e^{(1)} D_2 (T_{-i_x}(\sz_e))) = T_k(\sz_d)
			\end{align}
			
			Convolutional $F_d^{(0)}$ combines $\sz_d$ and skip connection $\sz_e$ to generate $y$. Since $\tsz_e = T_k(\sz_e)$ and $\tsz_d = T_k(\sz_d)$, this results in $\ty =T_k(y)$. We can similarly prove the result for even shift $k$ by using $\ti_x = i_x$.  
			The case of $L>1$ follows by induction.

		\end{proof}

		\begin{figure}[t]
			\includegraphics[width=1\linewidth]{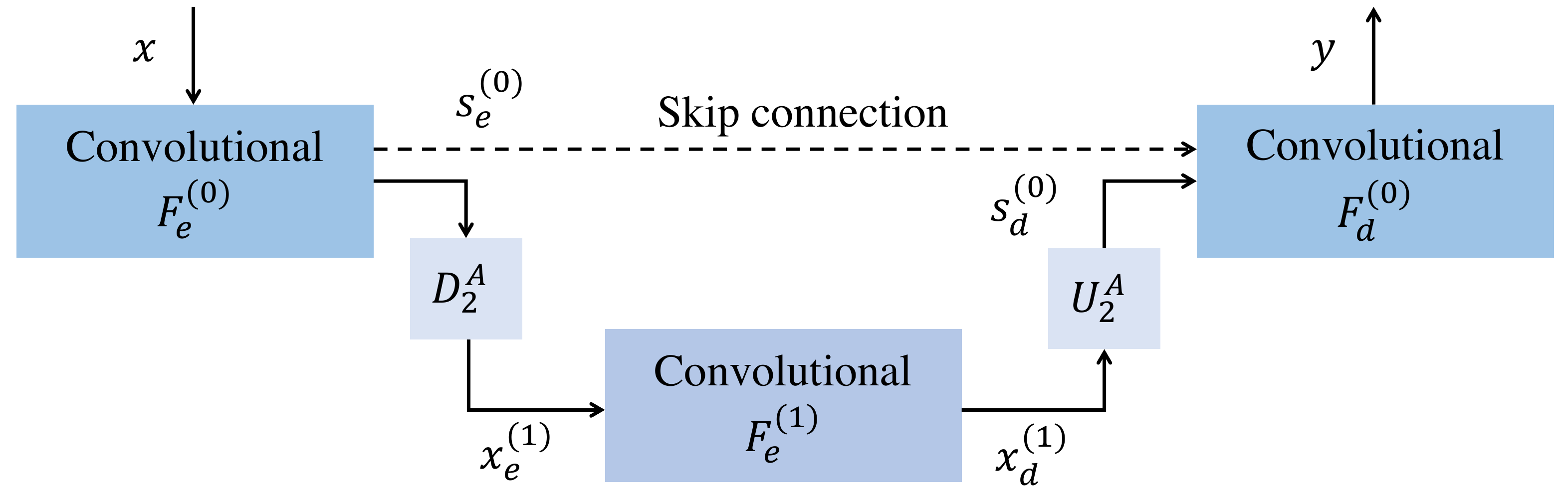}
			\centering
			\caption{U-Net containing a single APS-D and APS-U layer. }
			\label{fig:single_layer_unet_aps_d_u}
		\end{figure}

The proofs of Propositions \ref{prop:shift_eq_UD} and \ref{prop:unet_prop} rely on the fact that $\ti_x = 1-i_x$ when $k$ is odd. This condition is true as long as APS-D performs downsampling consistently. However, as pointed in \cite{chaman2020truly}, in the event when two polyphase components have exactly identical $l_p$ norms, this might fail to occur. However, in theory, if we assume polyphase components to be drawn from a continuous distribution, this is an event of probability zero. It is also very less likely in practice. As suggested in \cite{chaman2020truly}, one could avoid this rare event by selecting polyphase components using methods more robust than simple norm maximization.

		\begin{table*}[t]
			\centering
			\begin{tabular}{c|c c|c c |c  c |c  c |c  c  }
				\toprule
				\multicolumn{5}{c|}{\textbf{Equivariance metrics}}   &  \multicolumn{6}{c}{\textbf{Reconstruction metrics (unshifted)}}\\
				\hline
				\multirow{2}{*}{\textbf{Model}}  &  \multicolumn{2}{c|}{\textbf{NMSE}}    &        \multicolumn{2}{c|}{\textbf{SSIM}}      &  \multicolumn{2}{c|}{\textbf{NMSE}}    &   \multicolumn{2}{c|}{\textbf{PSNR}}         &   \multicolumn{2}{c}{\textbf{SSIM}} \\
				
				&  PD    &    PDFS  &        PD      &    PDFS   &   PD    &   PDFS     &  PD    &  PDFS  &  PD  &   PDFS \\
				\hline 
				Baseline &   0.0014   &   9.56e-4          &   0.9965     &  0.9975     &   0.016    &   0.053     &  33.83    &  29.92  &  0.8093  & 0.6301\\
				
				Baseline + DA &  1.37e-4  &	1.26e-4	&	0.9987	&	0.9990	&		0.016	&	0.053	&	33.58		&	29.86		&	0.8034		&	0.6272	\\	
							
				\hline 
				LPF-2 &  7.43e-4    &    5.34e-4     &      0.9978     &   0.9985    &     0.016  &   0.052     &  33.93    &  29.96  & 0.8122   & 0.6323 \\
							 
				LPF-3 &  4.97e-4    &   4.13e-4      &       0.9984    &   0.9988    &   \textbf{0.016}    &    \textbf{0.052}    &   \textbf{33.95}   & \textbf{29.96}   &   \textbf{0.8125} &\textbf{ 0.6325}\\
							
				LPF-5 &   5.05e-5   &    3.47e-5     &     0.9997      &   0.9998    &    0.018   &    0.055    &   33.19   &  29.76  &  0.7951  &  0.6225\\
				
				\hline 
				
				APS  &    \textbf{1.21e-7}  &    \textbf{7.37e-15}         &  \textbf{1.0  }    &   \textbf{1.0}    &   0.017    &   0.054     &  33.4    &  29.79  &  0.8013  &  0.6244 \\
				APS-2  &  \textbf{1.25e-7}    &  \textbf{5.86e-8}       &    \textbf{1.0}     &   \textbf{1.0}    &  0.017     &    0.054    &  33.51    &  29.83  &  0.8023  & 0.6255\\
				APS-3  &   \textbf{3.10e-7 }  &   \textbf{1.36e-7}          &   \textbf{1.0}     &   \textbf{1.0}    &   \textbf{0.016}    &   \textbf{0.052}     &  \textbf{33.95}    &  \textbf{29.96}  &  \textbf{0.8124 } & \textbf{0.6325} \\
				APS-5  &    \textbf{6.49e-7}  &    \textbf{2.74e-7}     &     \textbf{1.0 }    &   \textbf{1.0}   &   0.016    &     0.054   &   33.87   & 29.88   &  0.8088  & 0.6282\\
				\bottomrule

				%		&      &         &         &        &        &       &       &        &      &    &    & \\
			\end{tabular}
			\caption{Equivariance and reconstruction metrics obtained with different variants of U-Net on fastMRI validation set.}
			\label{tab:recon_and_equivariance}
			
		\end{table*}

		\begin{table*}[t]
			\centering
			\begin{tabular}{c|c |c |c |c |c | c |c |c | c }
				\toprule
				Model  &  Baseline    &		APS	&   LPF-2    &	APS-2	&  LPF-3     &		APS-3		&    LPF-5   	&	APS-5   & Baseline + DA			\\       
				\hline
				$(\Delta \text{PSNR})$  &   4.03  &  \textbf{8.6e-4}  &   2.58  &  \textbf{5.87e-4}  &  2.36  &  \textbf{3.81e-3}  &0.062  &   \textbf{3.69e-3}  &   0.037\\
				
				\bottomrule

			\end{tabular}
			
			\caption{Worst case decline in PSNR of MRI reconstructions caused by randomly shifting the images in fastMRI validation set.}
			\label{tab:worst_image_psnr_change}
			
		\end{table*}

		\begin{table}
			\centering
		\begin{tabular}{c | c    c   }
			\toprule
	%		\multicolumn{3}{c|}{\textbf{Evaluated on ImageNet}}   &  \multicolumn{2}{c}{\textbf{Evaluated on CT}}\\
	%		\hline
			Model &  NMSE  &  SSIM  \\
			\hline
			Baseline &8.67e-3 & 0.9722\\
			Baseline+DA &1.87e-3&0.9882\\
			\hline  
			LPF-2& 4.65e-3 & 0.9816 \\
			LPF-3&3.20e-3& 0.9861\\
			LPF-5& 4.40e-4& 0.9979\\
			\hline  
			APS&\textbf{3.09e-14} & \textbf{1.0}\\
			APS-2&\textbf{3.20e-14}& \textbf{1.0}\\
			APS-3&\textbf{3.39e-7} & \textbf{1.0}\\
			APS-5&\textbf{2.39e-6}& \textbf{1.0}\\
			
			\bottomrule

		\end{tabular}
		\caption{Equivariance metrics for networks trained on fastMRI training set but evaluated on ImageNet validation set.}
		\label{tab:mri_imagenet_out_of_dist}
	\end{table}

				\begin{figure*}[t]
			\includegraphics[width=0.9\linewidth]{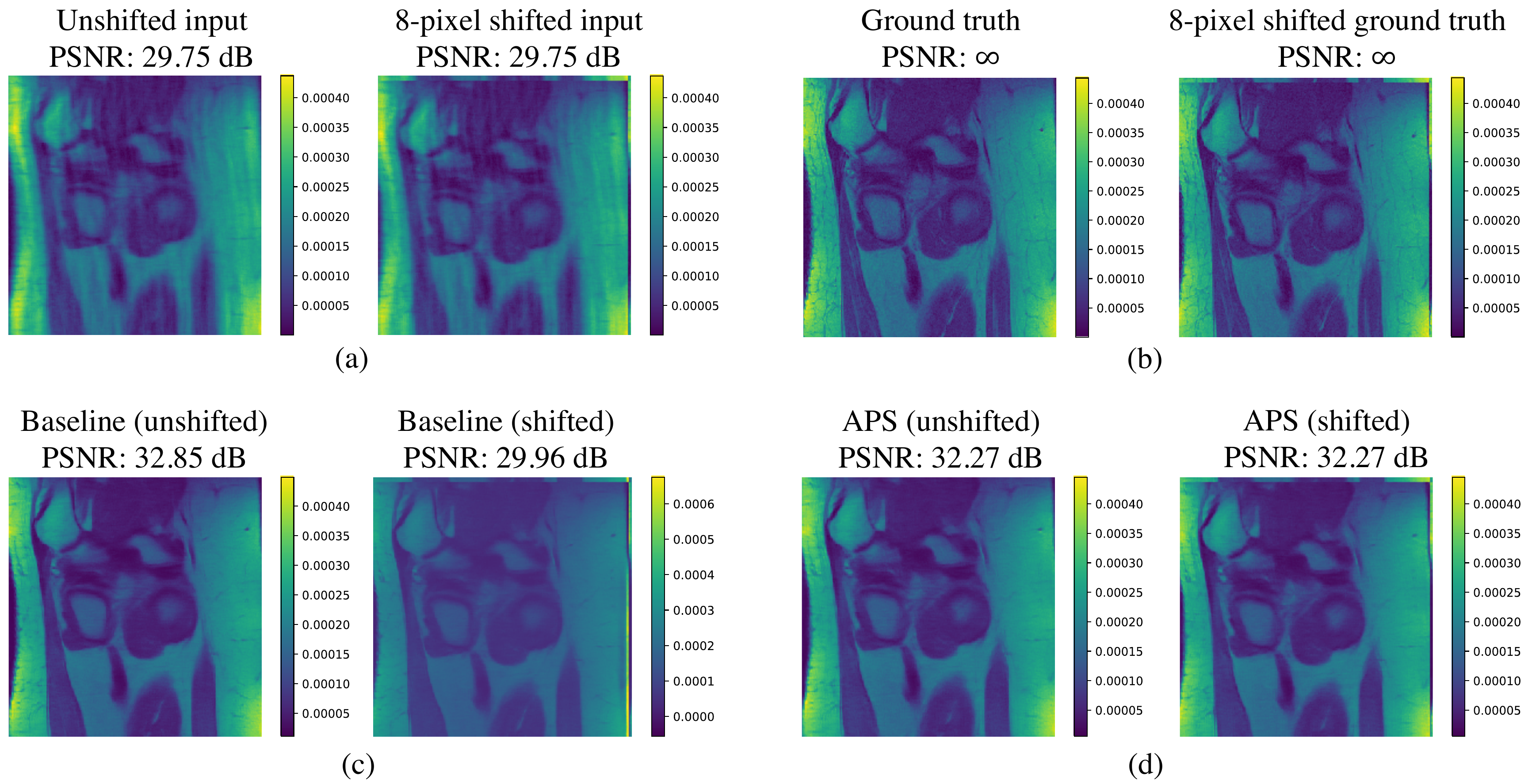}
			\centering
			\caption{Decline in PSNR of MRI reconstruction caused by shift in U-Net's input. (a) Input to the network and its (8, 5) pixel shifted version. (b) Ground truth. (c) Shifting the input to baseline U-Net results in a change in the pixel intensity distribution of its output and decline in PSNR. (d) Output of U-Net containing APS-D/U layers is highly robust to shifts. }
			\label{fig:visualize_psnr_decline}
		\end{figure*}

		\begin{table*}
			\centering
			\begin{tabular}{c | c  |  c   |     c   |c   |c   }
				\toprule
				\multicolumn{3}{c|}{\textbf{Equivariance metrics}}   &  \multicolumn{3}{c}{\textbf{Reconstruction metrics}}\\
				\hline
				\textbf{Model}  &  \textbf{NMSE}   &  \textbf{SSIM}   &   \textbf{NMSE}   &   \textbf{PSNR}  &  \textbf{SSIM} \\
				\hline
				Baseline  &     1.2e-4   &    0.9961  &   0.0056  &   34.76   &   0.8309  \\  
				Baseline+DA &  1.97e-5  &  0.9989  &   0.0056  &   34.7   & 0.8252 \\
				\hline
				
				LPF-2 &  2.95e-5  &   0.9991  &   \textbf{0.0052}   &   \textbf{35.35}    &  \textbf{0.836}  \\
				LPF-3 &  \textbf{2.83e-7 } &  \textbf{1.0}   &  0.0059   &   34.31  &   0.8191  \\
				LPF-5 & 1.67e-06 &  0.9999  &  0.0058  &  34.43  &  0.8239 \\
				\hline
				APS  &   \textbf{3.55e-8} &  \textbf{1.0}  &  0.0055  &  34.84  &   0.8214\\
				APS-2 & \textbf{1.53e-8}  &  \textbf{1.0} &\textbf{ 0.0052}  &  \textbf{35.31}  &  \textbf{0.8334}\\
				APS-3 &  \textbf{6.495e-10} &  \textbf{1.0} &  0.0054  & 34.95  & 0.8206\\
				APS-5 &  	\textbf{8.76e-9	}	&	\textbf{1.0	}	&	0.0056	&	34.69	&	0.8308		\\
				
				\hline

			\end{tabular}
			\caption{Equivariance and reconstruction metrics evaluated over LoDoPaB-CT dataset. }
			\label{tab:ct_circular_recon_and_equivariance}
		\end{table*}

		\section{Experiments}
		\label{sec:exps}
		
		We evaluate the performance of U-Net on MRI and CT reconstruction tasks with conventional (baseline) and APS-D/U sampling layers. For the two tasks, zero filled and filtered back projected images respectively are provided as input to the U-Net which then generates the final reconstructions. We also combine both the baseline and APS layers with anti-aliased filters of size $2\times2$, $3\times3$ and $5\times5$ similar to \cite{Zhang19}. Baseline models with low pass filters of size $j\times j$ are labelled as LPF-j, and those containing APS are denoted by APS-j. The baseline U-Net's encoder contains filters [64, 128, 256, 512, 1024]  with 4 strided-maxpool layers, and the decoder uses transposed convolutions for upsampling. Similar architecture is used for the LPF and APS based U-Net variants except that their stride layers are replaced by anti-aliased and APS based sampling respectively. All the networks were trained with MSE loss function and without any random shifts, unless mentioned otherwise. Networks trained with random shifts are denoted by DA. Further details on training and implementation are available in the code provided in \url{https://github.com/achaman2/truly_shift_invariant_cnns}.
		
		In addition to downsampling, CNNs can lose shift equivariance/invariance due to boundary effects as well \cite{chaman2020truly}. Therefore, to separate the impact of boundary artifacts and sampling we use circular padded convolutions and evaluate the networks for equivariance with circular shifts.

		We compared the models on two categories of metrics.
		\begin{itemize}
			\item \textbf{Equivariance metrics}: For inputs $x$ and $T_k(x)$ to a U-Net $G$, we use the SSIM and NMSE between $T_k(G(x))$ and $G(T_k(x))$ averaged over the dataset to evaluate shift equivariance of the U-Net. 
			
			We also examine the possible decline in PSNR ($\Delta PSNR$) of image reconstructions caused by shifting the U-Net's input.
			
	%		We also evaluate the worst case decline in PSNR caused by using randomly shifted input images to the U-Net. 
			
			\item \textbf{Reconstruction metrics}: To ensure that equivariance gains do not cause any sacrifice in reconstruction performance, we measure NMSE, PSNR and SSIM of the reconstructions for unshifted images.
			
		\end{itemize}

		\subsection{MRI reconstruction}
		We train and evaluate different variants of U-Net on FastMRI single coil knee reconstruction task \cite{zbontar2018fastMRI}. The networks were trained on a dataset containing $34742$ images and evaluated on the validation set with $7135$ images. The images were of size $320\times 320$.  Equivariance metrics were evaluated for each image with random shifts between $-16$ to $16$ and an average was computed over the 2 partitions provided in the dataset—`PDFS' and `PD'. Similar to \cite{zbontar2018fastMRI}, the reconstruction and equivariance metrics were computed over the provided image volumes rather than individual images. 
		
		Table \ref{tab:recon_and_equivariance} shows that models which use APS layers exhibit orders of magnitude lower equivariance errors than all other downsampling variants while still observing reconstruction performance comparable to baseline. In fact, they even surpass networks trained with random shifts (DA) on equivariance metrics by a large margin. 
		
		In addition to the equivariance metrics averaged over the entire dataset, we also assess worst case metrics, i.e. we shift each image in the validation set with 10 different random shifts and measure the worst absolute change in PSNR of the corresponding reconstructions obtained from U-Nets with different sampling modules. The worst possible decline for any image in the dataset is reported in Table \ref{tab:worst_image_psnr_change}. The results indicate that there exist shifts which can significantly impact the PSNR of reconstructions with baseline and LPF models. Networks with APS on the other hand are significantly more robust to shifts. An example illustration is provided in Fig. \ref{fig:visualize_psnr_decline}. We can observe in Fig. \ref{fig:visualize_psnr_decline}(c) that an (8, 5) pixel shift in the baseline network's input changes the pixel intensity distribution of its output and results in a decline in its PSNR with respect to the shifted ground truth. Consequently, the new reconstruction of the network is not a shifted version of its previous output. In contrast, as shown in Fig. \ref{fig:visualize_psnr_decline}(d), the two reconstructions obtained using APS U-Net are (8, 5) pixel shifted versions of each other.
		
%		The pseudo-inverse reconstruction of a knee and its (8, 5) pixel shifted version shown in Fig. \ref{fig:visualize_psnr_decline}(a) were provided as input to a baseline and APS U-Net. Fig. \ref{fig:visualize_psnr_decline}(b) contains the corresponding ground truth. We can observe that shifting the input to the baseline network changes pixel intensity distribution in the output, which is consequently not a shifted
			
		\subsubsection{Out-of-distribution equivariance}
		 In image classification, gains in shift invariance obtained from data augmentation and anti-aliasing are known to not extend well on out-of-distribution images \cite{Azulay_Weiss, chaman2020truly}. Here, we observe a similar phenomenon for shift equivariance with U-Net. We take networks trained on the fastMRI dataset and evaluated equivariance metrics on the first 1000 images from ImageNet validation set \cite{deng2009imagenet}. Table \ref{tab:mri_imagenet_out_of_dist} shows that SSIM equivariance metric on ImageNet for baseline network trained with data augmentation is $0.9882$, when it was $0.9990$ on the fastMRI dataset. We observe a similar decline for anti-aliased models as well. On the other hand, APS continues to provide SSIM of $1.0$ and orders of magnitude lower NMSE on the ImageNet dataset as well.The gains in shift equivariance provided by APS-D/U are therefore far more generalizable to out-of-dataset distributions in comparison to data augmentation and anti-aliasing.

		\subsection{CT reconstruction}
		\label{sec:ct_recon}
		We trained U-Net with different downsampling variants on the LoDoPaB-CT dataset \cite{leuschner2019lodopab} to perform CT reconstruction. The networks were trained on a dataset containing 35820 images and evaluated on the test set with 3553 images. We crop each image in the training and test set to size $352\times 352$ to ensure feature maps with even dimensions inside the U-Net. This was done to avoid boundary artifacts that arise when downsampling an odd length signal and its circular shifted version \cite{chaman2020truly}. Table \ref{tab:ct_circular_recon_and_equivariance} shows that similar to MRI reconstruction, networks with APS-D/U layers outperform the other models on shift equivariance, while performing comparably on reconstruction performance for unshifted images.

		\section{Conclusions}
		\label{sec:conclusions}
		Convolutional neural networks lose shift equivariance due to the presence of downsampling layers. While classical methods like data augmentation and anti-aliasing can improve shift equivariance on average, we show that they are not effective against all shifts. In addition, equivariance gains obtained with these methods are limited by the action of non-linear activations and do not necessarily extend well to image patterns not seen during training. In this work, we propose adaptive polyphase upsampling (APS-U) and combine it with our recently proposed adaptive polyphase downsampling (APS-D) scheme to enable perfect shift equivariance in symmetric encoder-decoder CNNs. Using experiments on MRI and CT reconstruction with U-Net architecture, we show that our approach significantly outperforms prior methods in improving translation equivariance. We also observe that the equivariance gains extend to out-of-distribution images and do not cause any sacrifice in reconstruction performance.

\section*{Acknowledgement}
This research was supported by the European Research Council Starting Grant 852821—SWING. Numerical experiments were partly performed at sciCORE (\href{http://scicore.unibas.ch/}{http://scicore.unibas.ch/}) scientific computing center at University of Basel. We also utilized computational resources supported by the National Science Foundation’s Major Research Instrumentation program, grant \#1725729, as well as the University of Illinois at Urbana-Champaign.

		\bibliographystyle{IEEEbib}
		\bibliography{ourbib}
		
	\end{document}